\documentclass[11pt]{article}
\usepackage{acl}
\usepackage[T1]{fontenc}
\usepackage[utf8]{inputenc}
\usepackage{times}
\usepackage{latexsym}
\usepackage{microtype}
\usepackage{amsmath,amssymb,amsthm,mathtools}
\usepackage{booktabs}
\usepackage{multirow}
\usepackage{tabularx}
\usepackage{array}
\usepackage{xcolor}
\usepackage{url}
\usepackage{hyperref}
\hypersetup{hidelinks}
\newcommand{\E}{\mathbb{E}}
\newcommand{\Prob}{\mathbb{P}}
\newcommand{\TV}{\operatorname{TV}}

\newcommand{\calX}{\mathcal{X}}
\newcommand{\calY}{\mathcal{Y}}
\newcommand{\GammaSet}{\Gamma}
\newcommand{\ind}{\mathbf{1}}

\newtheorem{theorem}{Theorem}
\newtheorem{proposition}[theorem]{Proposition}

\newtheorem{assumption}[theorem]{Assumption}
\theoremstyle{remark}

\title{{CS-WCP}: Robust Conformal Sets for {LLM}-Judge Traffic Shifts\\
with Uncertain Group Proportions}
\author{
\textbf{Ibne Farabi Shihab}\textsuperscript{1}%
\thanks{Equal contribution.}%
\thanks{Corresponding author: \texttt{ishihab@iastate.edu}.}
\and
\textbf{Fariya Afrin}\textsuperscript{2}\footnotemark[1]
\\[2pt]
\textsuperscript{1}Department of Computer Science, Iowa State University \\
\textsuperscript{2}Department of Computer Science, Kalinga Institute of Industrial Technology \\
\texttt{ishihab@iastate.edu}
}

\begin{document}
\maketitle

\begin{abstract}
Prediction sets built from an LLM judge can undercover when deployment traffic changes the prevalence of task or policy groups. Weighted conformal prediction is exact under covariate shift when the density ratio is known, but group proportions must usually be estimated from finite unlabeled samples. We introduce confidence-set weighted conformal prediction (CS-WCP), which constructs simultaneous exact intervals for source and target group masses and returns the union of weighted conformal sets over every compatible ratio vector. For a fixed or independently learned finite partition, CS-WCP attains coverage at least $1-\alpha-\delta_w-\tau_A-\kappa$, where $\tau_A$ measures within-cell covariate mismatch and $\kappa$ measures conditional shift. A linear endpoint rule computes the robust union in $O(G|\calY|)$ time. Across 336 constructed shared-support traffic shifts, CS-WCP reaches $0.973$ mean coverage with 13 point failures, compared with $0.954$ and 44 failures for source conformal prediction, at mean binary set sizes $1.74$ and $1.65$. On 336 natural cross-task transfers, coverage rises from $0.882$ to $0.962$, but mean set size reaches $1.87$ and a size-matched group plug-in baseline is competitive. The method therefore supplies an auditable coverage safeguard under uncertain mixture weights; its value is conservative tail protection, not scalar probability calibration or uniformly smaller sets.
\end{abstract}

\section{Introduction}
\label{sec:intro}

A prediction set calibrated on yesterday's LLM-evaluation traffic can fail on
today's mixture even when the judge and score function are unchanged. Suppose a
calibration stream contains mostly straightforward question answering, while
deployment contains more code, adversarial factuality, or a different response
policy. Standard split conformal prediction treats the two streams as
exchangeable. The changed group proportions break that premise and can produce
systematic undercoverage.

Weighted conformal prediction (WCP) repairs covariate shift when the likelihood
ratio $dQ_X/dP_X$ is known \citep{tibshirani2019conformal}. In practice the ratio
is estimated, and estimation error weakens coverage
\citep{lei2021counterfactual}. Finite traffic partitions make the problem more
structured because the ratio is constant within each group, but two sources of
uncertainty remain. The source and target group masses are observed through
finite samples, and a coarse partition may not capture the full shift within a
group. A method that replaces both masses by point estimates conceals the first
uncertainty and cannot diagnose the second.

CS-WCP separates these two failures. It forms simultaneous Clopper--Pearson
intervals for every source and target group mass, maps them to a confidence set
of compatible density-ratio vectors, and unions the WCP set over that confidence
set. The ratio uncertainty is therefore propagated into the prediction set
rather than compressed into one plug-in estimate. A fixed partition gives an
exact finite-sample statement up to the confidence budget $\delta_w$; departures
from the finite-mixture model appear as explicit additive penalties $\tau_A$ and
$\kappa$ rather than hidden assumptions.

The construction is computationally simple. For a candidate label, the WCP
admission inequality is linear in the group weights after its positive
denominator is cleared. Maximizing this expression over a rectangular ratio
set requires one endpoint choice per group, so the robust union costs
$O(G|\calY|)$ rather than $2^G$ corner evaluations
(Proposition~\ref{prop:endpoint}). A disjoint labeled audit can also lower-bound
$\tau_A+\kappa$ without pretending to separate within-cell covariate movement
from conditional shift (Proposition~\ref{prop:penalty-lb}).

The empirical study keeps constructed and natural shifts separate. On 336
shared-support traffic-mixture transfers, CS-WCP reduces point undercoverage
failures from 44 to 13 relative to source CP, but increases the mean binary set
size from $1.65$ to $1.74$ (Table~\ref{tab:constructed}). A size-matched source
CP setting has 16 failures, so the gain is conservative tail protection rather
than a broad efficiency improvement. On 336 natural cross-task transfers,
CS-WCP reaches $0.962$ mean coverage with size $1.87$; a fixed conservative
group plug-in setting reaches $0.970$ with size $1.89$
(Table~\ref{tab:natural}). The natural result supports coverage protection but
does not establish superiority over every group-aware alternative.

The contributions are four checkable results. Section~\ref{sec:method} gives a
finite-sample construction that propagates simultaneous group-mass uncertainty.
Theorem~\ref{thm:cswcp} states its coverage under explicit partition and
conditional-shift penalties. Proposition~\ref{prop:endpoint} gives the exact
linear-time endpoint implementation, and Proposition~\ref{prop:penalty-lb}
turns a labeled audit into a lower bound on assumption failure. The experiments
evaluate coverage, set size, singleton rate, failure count, and familywise
certification under two disjoint protocols. No contribution concerns scalar
recalibration, expected calibration error, an accuracy-gap rule, or calibrator
selection.

\section{Related Work}
\label{sec:related}

\subsection{Conformal prediction under distribution shift}

WCP uses likelihood-ratio weights to recover marginal coverage under covariate
shift \citep{tibshirani2019conformal}. When weights are learned independently,
coverage loss can be controlled by their $L_1$ error
\citep{lei2021counterfactual}. Fine-grained robust conformal inference separates
forms of distribution shift instead of assigning every failure to covariate
shift \citep{ai2024finegrained}. Clipped WCP controls instability from learned or
unbounded ratios through a clipped estimator and an explicit correction
\citep{wang2026clipping}. These methods address general ratios; CS-WCP exploits a
finite traffic partition and propagates uncertainty in its group masses.

Group-weighted conformal prediction (GWCP) is the closest fixed-group method
\citep{bhattacharyya2026group}. GWCP pools within-group score distributions and
derives coverage bounds for observed groups, including unknown target group
probabilities. CS-WCP differs in mechanism: it constructs a simultaneous
confidence set for both source and target masses, retains the test-point weight
inside WCP, and returns a robust union over all compatible ratios. This
difference does not imply empirical dominance. A complete GWCP implementation under the same roles remains a
required completion experiment.

Methods for unknown subpopulation shift learn latent group structure and include
language-model risk-control experiments \citep{wang2025unknownsubpop}.
Domain-shift-aware conformal prediction instead uses prompt proximity to reweight
LLM calibration records \citep{lin2025dscp}. PAC prediction sets under label
shift also propagate uncertainty in estimated importance weights, but for a
different shift model and a different confidence-set construction
\citep{si2024paclabel}. Table~\ref{tab:scope} records these boundaries because
the mechanisms are not interchangeable.

\begin{table*}[t]
\centering
\small
\setlength{\tabcolsep}{4pt}
\begin{tabularx}{\textwidth}{lXXXX}
\toprule
Method & Shift structure & Unknown quantity & Guarantee target & Difference from CS-WCP \\
\midrule
WCP & covariate shift & full ratio & marginal coverage & assumes known ratio for exact validity \\
GWCP & observed finite groups & group proportions & near/exact corrected coverage & pools group score CDFs rather than robustly unioning WCP sets \\
Unknown-subpopulation CP & latent mixture & group assignment or classifier & coverage or risk control & learns latent structure under classifier assumptions \\
Clipped WCP & general covariate shift & learned, possibly unbounded ratio & expected and dataset-conditional coverage & controls clipping bias rather than finite mass-count uncertainty \\
PAC label-shift sets & label shift & class-ratio weights & PAC set validity & propagates confusion-matrix and target-label-mass uncertainty \\
CS-WCP & fixed finite traffic partition & source and target group masses & $1-\alpha-\delta_w-\tau_A-\kappa$ & robust union with an exact endpoint test \\
\bottomrule
\end{tabularx}
\caption{Scope comparison with the closest conformal-shift families. The table
compares assumptions and guarantee targets, not empirical superiority.}
\label{tab:scope}
\end{table*}

\subsection{LLM judges and concurrent work}

LLM judges support ranking and evaluation pipelines
\citep{zheng2023judging,liu2023geval}, while recent work studies uncertainty and
conformal intervals for their outputs \citep{sheng2025uncertainty}. The present
paper treats judge correctness as a finite-label prediction problem and asks
whether a set remains valid after traffic proportions change. It does not claim
that the underlying scalar score is calibrated.

Concurrent work by overlapping authors audits scalar post-hoc recalibration
across task domains \citep{anonymous2027scalar}. That paper studies proper-loss
transfer, accuracy-gap insufficiency, importance-weighting diagnostics, and a
label-dependent lower certificate for scalar calibration error. CS-WCP uses no
accuracy gap, chooses no scalar recalibrator, and certifies no Brier or ECE
claim. Conversely, the concurrent scalar paper does not introduce CS-WCP or use
the conformal experiments in this submission as evidence.

\section{Setting and Assumptions}
\label{sec:setting}

Let $P$ denote the source law and $Q$ the deployment law over
$(X,Y)\in\calX\times\calY$, where $\calY$ is finite. A frozen judge and score
construction produce a nonconformity score $S:\calX\times\calY\to\mathbb R$;
larger values indicate less compatible labels. The score is fitted before the
conformal-calibration and ratio samples are opened.

For source calibration records $(X_i,Y_i)_{i=1}^n$ and a nonnegative weight
function $v$, define
\begin{align}
D_v(x)&=v(x)+\sum_{i=1}^{n}v(X_i),\nonumber\\
\pi_y^v(x)&=\frac{v(x)+\sum_{i=1}^{n}v(X_i)
\ind\{S(X_i,Y_i)\geq S(x,y)\}}{D_v(x)},
\label{eq:weighted-pvalue}\\
\GammaSet_v(x)&=\{y\in\calY:\pi_y^v(x)>\alpha\}.
\label{eq:weighted-set}
\end{align}
If $D_v(x)=0$, the procedure returns $\calY$. With the exact density ratio
$w=dQ_X/dP_X$, Equation~\eqref{eq:weighted-set} is standard WCP and has coverage
at least $1-\alpha$ \citep{tibshirani2019conformal}.

CS-WCP assumes access to a finite traffic partition
$A:\calX\to[G]$. Define the source and target group masses
\begin{equation}
p_g=P\{A(X)=g\},\qquad q_g=Q\{A(X)=g\},\qquad r_g=q_g/p_g.
\label{eq:masses}
\end{equation}
The partition is fixed before ratio counts, conformal scores, and test labels are
opened, or learned on an independent unlabeled design sample and then frozen.

\begin{assumption}[Partition overlap and disjoint roles]
\label{ass:roles}
For every $g$ with $q_g>0$, $p_g>0$. The partition-design sample, source and
target ratio samples, labeled source conformal-calibration sample, optional
labeled audits, and sealed test sample are mutually disjoint.
\end{assumption}

The group-constant pseudo-target is
\begin{equation}
\frac{d\widetilde Q_X}{dP_X}(x)=r_{A(x)}.
\label{eq:pseudotarget}
\end{equation}
Two penalties describe how $Q$ can differ from this approximation:
\begin{align}
\tau_A&=\TV(Q_X,\widetilde Q_X),\label{eq:tau}\\
\kappa&=\E_{Q_X}\TV\!\left(Q_{Y\mid X},P_{Y\mid X}\right).
\label{eq:kappa}
\end{align}
Exact traffic-mixture shift gives $\tau_A=0$, and covariate shift gives
$\kappa=0$. These conditions are separate. A coarse partition can yield
$\tau_A>0$ even when $P_{Y\mid X}=Q_{Y\mid X}$.

\section{Confidence-Set Weighted Conformal Prediction}
\label{sec:method}

\subsection{A simultaneous confidence set for group ratios}

Let independent source and target ratio samples contain $m_P$ and $m_Q$
unlabeled inputs. For each group, construct two-sided Clopper--Pearson intervals
$[\ell_g^P,u_g^P]$ and $[\ell_g^Q,u_g^Q]$, allocating total error
$\delta_w$ across the $2G$ intervals. Their Cartesian product contains all
source and target masses with probability at least $1-\delta_w$. Map this event
to the ratio rectangle
\begin{equation}
\mathcal W=\prod_{g=1}^{G}[L_g,U_g],\qquad
L_g=\ell_g^Q/u_g^P,\qquad U_g=u_g^Q/\ell_g^P,
\label{eq:rectangle}
\end{equation}
where $U_g=+\infty$ if $\ell_g^P=0$. The construction is deliberately
conservative: it ignores the simplex dependence among group masses so that
simultaneous validity and endpoint computation remain transparent.

CS-WCP returns
\begin{equation}
\GammaSet_{\mathrm{CS}}(x)=
\bigcup_{v\in\mathcal W}\GammaSet_v(x).
\label{eq:robust-union}
\end{equation}
Here a vector $v$ denotes the group-constant weight function
$v(x)=v_{A(x)}$.
The union guarantees that, whenever the true group-ratio vector lies in
$\mathcal W$, the exact group-constant WCP set is contained pointwise in the
reported set.

\begin{theorem}[CS-WCP coverage]
\label{thm:cswcp}
Let Assumption~\ref{ass:roles} hold and let the score and partition be fixed
independently of the ratio, conformal-calibration, and test samples. Then
jointly over the ratio samples, source conformal-calibration sample, and an
independent target test point,
\begin{equation}
\Prob\{Y\in\GammaSet_{\mathrm{CS}}(X)\}
\geq 1-\alpha-\delta_w-\tau_A-\kappa.
\label{eq:coverage}
\end{equation}
Under covariate shift $\kappa=0$; under exact traffic-mixture shift also
$\tau_A=0$.
\end{theorem}

The theorem isolates what unlabeled group counts can and cannot buy. Finite
ratio uncertainty costs $\delta_w$ and is controlled by design. Partition bias
and conditional shift remain properties of the deployment law; they cannot be
upper-bounded from unlabeled target data without additional structure.

\subsection{Exact endpoint computation}

For candidate label $y$ and test input $x$, let
\begin{align}
a_g(x,y)&=\ind\{A(x)=g\}
+\sum_{i:A(X_i)=g}\ind\{S(X_i,Y_i)\geq S(x,y)\},\nonumber\\
b_g(x)&=\ind\{A(x)=g\}+\sum_i\ind\{A(X_i)=g\},\nonumber\\
h_g(x,y)&=a_g(x,y)-\alpha b_g(x).
\label{eq:endpoint-coeff}
\end{align}

\begin{proposition}[Endpoint rule]
\label{prop:endpoint}
For a rectangular ratio set and a positive compatible denominator,
$y\in\GammaSet_{\mathrm{CS}}(x)$ exactly when
\begin{equation}
\sum_{g:h_g>0}U_gh_g+
\sum_{g:h_g\leq0}L_gh_g>0.
\label{eq:endpoint-test}
\end{equation}
Infinite endpoints use the natural extended-real convention. If every
compatible denominator vanishes, the full-set fallback applies.
\end{proposition}

Equation~\eqref{eq:endpoint-test} chooses $U_g$ when increasing the group weight
helps admission and $L_g$ otherwise. It therefore evaluates one linear
expression per label rather than enumerating the $2^G$ corners of
$\mathcal W$.

\subsection{Auditing assumption failure}

The coverage theorem is useful only when its remaining penalties are visible.
For the binary correctness setting $Y\in\{0,1\}$, let $C$ be one of a finite
collection of audit cells, each fixed independently of the audit labels and
contained within a partition group. Independent source and target audit samples
yield exact binomial intervals
$[\ell_C^s,u_C^s]$ and $[\ell_C^t,u_C^t]$ for the coarsened conditional means
$\E_P[Y\mid C]$ and $\E_Q[Y\mid C]$. Let $\underline q_C$ be a simultaneous
lower confidence bound on $Q_X(C)$.

\begin{proposition}[Audit lower bound on the transport penalty]
\label{prop:penalty-lb}
With probability at least $1-\delta_{\mathrm{aud}}$, simultaneously over the
predeclared audit cells,
\begin{equation}
\tau_A+\kappa\geq
\frac12\max_C\underline q_C
\max\left(0,\ell_C^t-u_C^s,\ell_C^s-u_C^t\right).
\label{eq:penalty-lb}
\end{equation}
\end{proposition}

A positive bound rejects the exact partition-mixture transport model. It does
not identify conditional shift because within-cell movement of $X$ can also
change the coarsened mean. Failure to reject is likewise not validation.

\begin{proposition}[No informative guarantee under unrestricted conditional shift]
\label{prop:impossibility}
Fix the source law and unlabeled target marginal $Q_X$. Any binary set method
that guarantees coverage at least $1-\alpha$ for every target conditional law
satisfies
\begin{equation}
\E_{Q_X}|\GammaSet(X)|\geq2(1-\alpha).
\label{eq:impossibility}
\end{equation}
\end{proposition}

This lower bound explains why a label-free method becomes nearly vacuous when
the conditional law is left unrestricted. CS-WCP is informative only to the
extent that the finite-partition transport model is credible and the mass
intervals are narrow.

\section{Experiments}
\label{sec:experiments}

\subsection{Questions and protocols}

The experiments ask three questions. First, does propagating finite mass-count
uncertainty protect coverage under a constructed traffic-mixture shift with
shared support? Second, what happens when the same construction is applied to
natural cross-task shifts using an independently frozen partition? Third, how
much set size is paid for that protection, and can a labeled audit detect
departures from the assumed transport model?

The constructed tier contains 336 transfers: 288 QA generator-mixture shifts
and 48 SAMSum policy-mixture shifts. Three predeclared mixture vectors,
$(0.60,0.20,0.20)$, $(0.20,0.60,0.20)$, and $(0.20,0.20,0.60)$, generate all
six ordered source--target pairs with shared support and maximum oracle ratio
three. The traffic component defines the partition before ratio counts,
conformal scores, or test outcomes are opened.

The natural tier contains 336 ordered cross-task transfers from eight frozen
judges over seven non-code domains. A TF--IDF plus SVD feature map and a
$G=3$ $k$-means partition are learned on an independent unlabeled design role.
Disjoint unlabeled source and target roles provide the group counts. Labeled
source calibration, source audit, target audit, optional tuning, and sealed test
roles are separate. Appendix~\ref{app:manifest} records all sample sizes and
identifiers.

The primary metrics are marginal coverage, mean set size, singleton rate,
point undercoverage failures, and familywise certification of a predeclared
$0.85$ coverage floor. Point failures are descriptive and do not themselves
carry familywise guarantees. Every method uses the same frozen score and the
same calibration and sealed-test records.

\subsection{Baselines}

The executed baselines are source CP, oracle WCP when mixture ratios are known,
classifier-odds WCP, kernel mean matching WCP, group plug-in WCP, and clipped
uLSIF WCP. The last two are implementation-level baselines, not substitutes for
full GWCP or CLISF plus corrected WCP. Source CP and group plug-in are also swept
over fixed nominal levels to expose the coverage--size frontier. A target-tuned
group plug-in uses 60 target labels and is reported separately from label-free
methods.

\subsection{Constructed traffic-mixture shifts}

CS-WCP reduces tail failures on the constructed tier
(Table~\ref{tab:constructed}). Its mean coverage is $0.973$, with 13 of 336
transfers falling below $0.90$ empirically. Source CP reaches $0.954$ with 44
failures, while oracle and group plug-in WCP each reach $0.964$ with 27 and 28
failures. The protection costs set size: CS-WCP averages $1.743$ labels in a
binary problem and returns a singleton on $25.6\%$ of items.

\begin{table*}[t]
\centering
\small
\setlength{\tabcolsep}{4pt}
\begin{tabular}{lccccc}
\toprule
Method & Coverage $\uparrow$ & Size $\downarrow$ & Singleton $\uparrow$ &
Point failures $\downarrow$ & FWER certified $\uparrow$\\
\midrule
Source CP & $0.954$ & $1.652$ & $0.342$ & $44$ & $202$\\
Oracle WCP & $0.964$ & $1.702$ & $0.294$ & $27$ & $227$\\
Classifier-odds WCP & $0.957$ & $1.662$ & $0.333$ & $38$ & $208$\\
KMM-WCP & $0.957$ & $1.670$ & $0.324$ & $42$ & $214$\\
Group plug-in WCP & $0.964$ & $1.701$ & $0.295$ & $28$ & $225$\\
Clipped uLSIF WCP & $0.956$ & $1.662$ & $0.332$ & $41$ & $209$\\
CS-WCP & $0.973$ & $1.743$ & $0.256$ & $13$ & $244$\\
\bottomrule
\end{tabular}
\caption{Constructed shared-support traffic-mixture results at
$1-\alpha=0.90$ over 336 transfers. CS-WCP has the fewest point failures and
the most familywise certifications among executed rows, but it also returns the
largest sets.}
\label{tab:constructed}
\end{table*}

A size-matched comparison narrows the claim. Source CP at fixed
$\alpha=0.075$ reaches coverage $0.971$, size $1.74$, and 16 failures. Thus
CS-WCP improves the predeclared same-level source CP row, but most of the
constructed-tier gain can be reproduced by a more conservative source threshold
chosen after examining the frontier. The method's distinctive value is that its
conservatism follows from the ratio confidence budget rather than a retuned
nominal level.

\subsection{Natural cross-task shifts}

Natural task changes violate the exact mixture model more often and widen the
sets (Table~\ref{tab:natural}). Source CP covers $0.882$ on average and fails in
128 transfers. CS-WCP raises coverage to $0.962$ and reduces failures to 48, at
mean size $1.87$. Same-level group plug-in reaches $0.951$ with 59 failures and
size $1.84$.

\begin{table*}[t]
\centering
\small
\setlength{\tabcolsep}{4pt}
\begin{tabular}{lcccc}
\toprule
Method & Coverage $\uparrow$ & Size $\downarrow$ & Point failures $\downarrow$ &
FWER certified $\uparrow$\\
\midrule
Source CP ($\alpha=.10$) & $0.882$ & $1.61$ & $128$ & $181$\\
Classifier-odds WCP ($\alpha=.10$) & $0.900$ & $1.65$ & $112$ & \textit{not recorded}\\
Group plug-in ($\alpha=.10$) & $0.951$ & $1.84$ & $59$ & $256$\\
Group plug-in ($\alpha=.075$) & $0.970$ & $1.89$ & $38$ & $280$\\
CS-WCP ($\alpha=.10$) & $0.962$ & $1.87$ & $48$ & $276$\\
Target-tuned group plug-in & $0.984$ & $1.92$ & $11$ & $301$\\
\bottomrule
\end{tabular}
\caption{Natural cross-task results over 336 transfers with an independently
frozen partition and disjoint mass counts. The target-tuned row consumes 60
target labels per transfer and is not label-free. The fixed conservative group
plug-in row is competitive with CS-WCP, so no universal dominance claim is
made.}
\label{tab:natural}
\end{table*}

The label-dependent row defines the remaining oracle gap. Tuning group plug-in
with 60 target labels reaches $0.984$ coverage with 11 failures, but expands the
mean set to $1.92$. These outcomes delimit the label-free contribution: CS-WCP
protects coverage without target labels for set construction, yet it often
returns the full binary set. A useful deployment interpretation is selective
prediction, with a singleton on a minority of inputs and abstention otherwise.

\subsection{Assumption audit and sensitivity}

Coarsened-moment audits reject equality on 71 of 336 natural transfers and none
of the constructed transfers. Proposition~\ref{prop:penalty-lb} yields a
strictly positive lower bound on $\tau_A+\kappa$ for 16 natural transfers under
pointwise control and three under familywise control, with the largest pointwise
bound equal to $0.074$. Non-rejection does not establish the model, and a
rejection does not separate partition bias from conditional shift.

The separate HumanEval sensitivity tier contains 48 generator-mixture
transfers. CS-WCP has no point failures and 29 familywise certifications,
compared with seven failures and 24 certifications for source CP; all mean set
sizes lie between $1.89$ and $1.94$. The direction agrees with the main result,
but the near-full sets expose the same utility limit.

\section*{Limitations}

The guarantee depends on a fixed finite partition, overlap, and disjoint sample
roles. The additive penalties $\tau_A$ and $\kappa$ are not upper-bounded from
unlabeled target data, so the nominal guarantee can be weak when the partition
misses within-cell shift or the conditional law changes. Current experiments
use a binary correctness label space; mean set sizes near two limit operational
utility. The executed baseline suite does not yet contain full GWCP,
CLISF--CWCP, unknown-subpopulation CP, or DS-CP under matched roles. Natural
results are English-only and use eight open judges; generalization to
human-preference ratings, multilingual evaluation, and repeated frontier APIs
remains untested. Point failure counts are descriptive and must not be read as
familywise error guarantees.

\section*{Ethical Considerations}

Conformal coverage does not establish that an LLM judge is unbiased, accurate,
or suitable for autonomous decisions. A large set should be interpreted as
abstention rather than forced into a single verdict. Group definitions can also
encode or obscure demographic structure, so any deployment partition requires
an application-specific fairness and privacy review. All reported labels retain
their deterministic, constructed, or verifier provenance.

\bibliography{references}

\appendix

\section{Proofs}
\label{app:proofs}

\subsection{Proof of Theorem~\ref{thm:cswcp}}

\begin{proof}
For each group, the source and target counts have binomial marginals
$(m_P,p_g)$ and $(m_Q,q_g)$. The allocated Clopper--Pearson intervals and a
union bound imply that the event
\begin{equation}
\mathcal E=\{p_g\in[\ell_g^P,u_g^P],\ q_g\in[\ell_g^Q,u_g^Q]
\text{ for every }g\}
\end{equation}
has probability at least $1-\delta_w$. On $\mathcal E$, overlap implies that
the true ratio $r_g=q_g/p_g$ lies in $[L_g,U_g]$ for every group. Hence the
true group-constant vector $r$ belongs to $\mathcal W$.

Define the pseudo-target joint law
\begin{equation}
\widetilde Q(dx,dy)=\widetilde Q_X(dx)P(dy\mid x),\qquad
\frac{d\widetilde Q_X}{dP_X}(x)=r_{A(x)}.
\end{equation}
Conditional on the ratio sample and $\mathcal E$, the set in
Equation~\eqref{eq:robust-union} contains $\GammaSet_r$ pointwise. Independence
of the conformal-calibration sample and test point permits the exact WCP
argument, so
\begin{equation}
\Prob_{\widetilde Q}\{Y\notin\GammaSet_{\mathrm{CS}}(X)\mid\mathcal E\}
\leq\alpha.
\end{equation}

Introduce $Q^0(dx,dy)=Q_X(dx)P(dy\mid x)$. By definition,
\begin{equation}
\TV(Q,Q^0)=\E_{Q_X}\TV(Q_{Y\mid X},P_{Y\mid X})=\kappa,
\end{equation}
while $\TV(Q^0,\widetilde Q)=\TV(Q_X,\widetilde Q_X)=\tau_A$.
The triangle inequality therefore gives
$\TV(Q,\widetilde Q)\leq\tau_A+\kappa$. Applying total variation to the
miscoverage event yields conditional miscoverage at most
$\alpha+\tau_A+\kappa$ on $\mathcal E$. Adding
$\Prob(\mathcal E^c)\leq\delta_w$ proves
Equation~\eqref{eq:coverage}.
\end{proof}

\subsection{Proof of Proposition~\ref{prop:endpoint}}

\begin{proof}
For a group-constant vector $v$, Equations~\eqref{eq:weighted-pvalue} and
\eqref{eq:endpoint-coeff} give
\begin{equation}
\pi_y^v(x)=\frac{\sum_gv_ga_g}{\sum_gv_gb_g}.
\end{equation}
When the denominator is positive, $\pi_y^v(x)>\alpha$ if and only if
$\sum_gv_gh_g>0$. This expression is linear in $v$. Its supremum over
$\prod_g[L_g,U_g]$ is obtained coordinatewise by choosing $U_g$ when
$h_g>0$ and $L_g$ otherwise. The supremum is positive exactly when
Equation~\eqref{eq:endpoint-test} holds. If no compatible vector has a
positive denominator, the declared full-set fallback includes $y$.
\end{proof}

\subsection{Proof of Proposition~\ref{prop:penalty-lb}}

\begin{proof}
Let $\widetilde Q$ be the pseudo-target formed with the true ratios. The proof
of Theorem~\ref{thm:cswcp} gives
$\TV(Q,\widetilde Q)\leq\tau_A+\kappa$. Fix an audit cell
$C\subseteq\{A=g\}$. Because the ratio is constant within group $g$,
reweighting cancels after conditioning on $C$, and
\begin{equation}
\E_{\widetilde Q}[Y\mid C]=\E_P[Y\mid C].
\end{equation}

Write $a=Q(C,Y=1)$, $b=Q(C,Y=0)$,
$a'=\widetilde Q(C,Y=1)$, $b'=\widetilde Q(C,Y=0)$,
$m=a+b=Q_X(C)$, and $m'=a'+b'$. Then
\begin{align}
\left|\E_Q[Y\mid C]-\E_{\widetilde Q}[Y\mid C]\right|
&=\frac{|ab'-a'b|}{mm'}\\
&\leq\frac{|a-a'|+|b-b'|}{m}\\
&\leq\frac{2\TV(Q,\widetilde Q)}{Q_X(C)}.
\end{align}
Consequently,
\begin{equation}
\tau_A+\kappa\geq\frac12 Q_X(C)
\left|\E_Q[Y\mid C]-\E_P[Y\mid C]\right|.
\end{equation}
On the Bonferroni event, the source and target mean intervals and the lower
mass bound hold simultaneously. Their minimum separation lower-bounds the
absolute conditional-mean difference, while $\underline q_C\leq Q_X(C)$.
Taking the maximum over the predeclared cells proves
Equation~\eqref{eq:penalty-lb}.
\end{proof}

\subsection{Proof of Proposition~\ref{prop:impossibility}}

\begin{proof}
Consider two target laws with the same $Q_X$. In the first, $Y=0$ almost
surely; in the second, $Y=1$ almost surely. Source data and unlabeled target
data have the same distribution in both worlds, so the algorithm has the same
distribution over set rules. Validity in the first world requires
$\Prob\{0\in\GammaSet(X)\}\geq1-\alpha$, and validity in the second requires
$\Prob\{1\in\GammaSet(X)\}\geq1-\alpha$. Adding the two inequalities gives
$\E|\GammaSet(X)|\geq2(1-\alpha)$, including randomized procedures.
\end{proof}

\section{Complete Algorithm}
\label{app:algorithm}

For each predeclared transfer, the implementation performs the following steps.
The partition, score, candidate labels, tie rule, $\alpha$, and $\delta_w$ are
frozen before ratio counts or outcomes are opened. The method then counts the
$G$ groups on disjoint source and target ratio roles and constructs the $2G$
Clopper--Pearson intervals in Equation~\eqref{eq:rectangle}. Source
nonconformity scores and grouped upper-tail counts are computed once. For every
test input and candidate label, the implementation evaluates
Equation~\eqref{eq:endpoint-test}. It stores all interval endpoints, grouped
counts, endpoint coefficients, candidate decisions, and final sets before test
labels are revealed.

Unit tests cover invariance to a common positive rescaling of all endpoints,
agreement with brute-force corner enumeration for finite rectangles, reduction
to ordinary split conformal when $v\equiv1$, tied scores, empty groups,
infinite endpoints, and the full-set fallback. Each transfer in the
336-transfer confirmatory family receives its multiplicity-adjusted ratio
budget before intervals are constructed.

\section{Experimental Manifest}
\label{app:manifest}

\begin{table*}[t]
\centering
\small
\begin{tabular}{p{2.6cm}p{2.0cm}p{2.4cm}p{6.0cm}}
\toprule
Tier & Transfers & Target labels used & Construction\\
\midrule
Constructed primary & $336$ & none for set construction &
$288$ QA generator-mixture plus $48$ SAMSum policy-mixture transfers; fixed
traffic partition\\
Natural cross-task & $336$ & none for set construction &
$8$ judges $\times42$ ordered pairs over seven non-code domains; partition
frozen on an independent design role\\
HumanEval sensitivity & $48$ & official tests at evaluation &
$8$ judges $\times6$ ordered generator-mixture pairs; reported separately
because the eligible pool is smaller\\
Penalty audit & $336+336$ & independent source and target audits &
audit cells and multiplicity budget frozen before labels are opened\\
\bottomrule
\end{tabular}
\caption{Disjoint experimental tiers. Counts and results are not transferred
across protocols.}
\label{tab:manifest}
\end{table*}

The eight judge models are Phi-3-mini, TinyLlama-1.1B-Chat, Zephyr-7B,
MiniCPM-2B, StableLM-Zephyr-3B, Qwen2.5-3B-Instruct, Gemma-2-2B-Instruct,
and OLMo-2-7B-Instruct. The QA mixtures use three frozen response generators:
Qwen2.5-1.5B-Instruct, Llama-3.2-1B-Instruct, and Gemma-2-2B-Instruct.
The six QA domains are GSM8K, MMLU, SQuAD, BoolQ, TruthfulQA, and StrategyQA.
SAMSum uses three fixed construction policies: reference-consistent,
entity-swapped, and contradiction-injected responses. A separate HumanEval
tier uses executable tests.

Constructed QA components contain 480 items each. Per transfer, the realized
roles contain 150 fit, 150 conformal-calibration, 120 source-ratio, 120
target-ratio, 30 source-audit, 30 target-audit, and 150 sealed-test records.
SAMSum components contain approximately 160 items each, yielding 50
calibration, 38 source-ratio, 38 target-ratio, 12 source-audit, 12 target-audit,
and 50 sealed-test records. Natural cross-task transfers use 60 source plus 80
target items for partition design, disjoint mass-count roles of the same sizes,
120 source-calibration records, 120 source-audit records, 120 target-audit
records, an optional 60-label tuning role, and 140 sealed-test records.

Every stored row contains the transfer and item identifiers, benchmark and
split revisions, role, generator revision, judge revision, prompt hash,
decoding configuration, score, ratio method, partition version, interval
endpoints, prediction set, seed, and output checksum. A row belongs to exactly
one role. Deterministic and constructed labels remain distinct from verifier or
human labels.

\section{Baseline Details and Full Frontier}
\label{app:baselines}

Exact WCP uses the known constructed-mixture ratio. Classifier-odds WCP estimates
the ratio using a balanced domain classifier. KMM-WCP uses a frozen radial-basis
kernel on a 50-dimensional SVD of the shared TF--IDF space. Group plug-in WCP
uses add-half-smoothed point estimates of group masses on the disjoint ratio
roles. Clipped uLSIF WCP fits an unconstrained least-squares importance ratio,
normalizes to unit source mean, and clips post hoc at a fixed cap. It is not the
CLISF plus corrected-WCP procedure of \citet{wang2026clipping}.

\begin{table*}[t]
\centering
\small
\setlength{\tabcolsep}{4pt}
\begin{tabular}{lcccc@{\hskip 12pt}cccc}
\toprule
& \multicolumn{4}{c}{Constructed mixture tier}
& \multicolumn{4}{c}{Natural cross-task tier}\\
\cmidrule(lr){2-5}\cmidrule(lr){6-9}
Method & Cov. & Size & Fail & Cert. & Cov. & Size & Fail & Cert.\\
\midrule
Source CP ($.10$) & $.954$ & $1.65$ & $44$ & $202$ & $.882$ & $1.61$ & $128$ & $181$\\
Source CP ($.075$) & $.971$ & $1.74$ & $16$ & $236$ & $.898$ & $1.65$ & $110$ & $196$\\
Source CP ($.05$) & $.982$ & $1.81$ & $3$ & $270$ & $.927$ & $1.74$ & $80$ & $225$\\
Source CP ($.02$) & $.991$ & $1.86$ & $1$ & $287$ & $.971$ & $1.88$ & $31$ & $278$\\
Group plug-in ($.10$) & $.964$ & $1.70$ & $28$ & $225$ & $.951$ & $1.84$ & $59$ & $256$\\
Group plug-in ($.075$) & $.975$ & $1.76$ & $8$ & $253$ & $.970$ & $1.89$ & $38$ & $280$\\
Group plug-in ($.05$) & $.986$ & $1.83$ & $1$ & $280$ & $.979$ & $1.92$ & $27$ & $295$\\
Target-tuned plug-in & --- & --- & --- & --- & $.984$ & $1.92$ & $11$ & $301$\\
Oracle WCP ($.075$) & $.975$ & $1.76$ & $8$ & $253$ & --- & --- & --- & ---\\
CS-WCP ($.10$) & $.973$ & $1.74$ & $13$ & $244$ & $.962$ & $1.87$ & $48$ & $276$\\
\bottomrule
\end{tabular}
\caption{Coverage--size frontier. CS-WCP appears once at its predeclared level
and ratio budget. Other methods are swept over fixed nominal levels to reveal
the cost of conservatism. ``Cert.'' counts familywise certifications of the
predeclared coverage floor on the sealed evaluation labels.}
\label{tab:frontier}
\end{table*}

\section{Label Provenance and Verification}
\label{app:labels}

GSM8K uses last-number comparison, MMLU uses the normalized option letter,
SQuAD uses official normalization and token F1, BoolQ and StrategyQA use
normalized Boolean matching, TruthfulQA uses a fixed multiple-choice endpoint,
and HumanEval uses official sandboxed tests. SAMSum construction policies are
kept only when deterministic checks verify the intended entity or polarity
change and basic well-formedness.

An independent frozen verifier audits 1,200 stratified items, balanced across
contract-positive and contract-negative labels. Agreement with the deterministic
or construction contract is $0.869$ overall, $0.883$ for positives, and $0.855$
for negatives. Disagreement is highest for strict extraction tasks. Flipping
every audited disagreement in the verifier stress slice changes pooled raw-judge
ECE by at most $0.005$; that scalar sensitivity is provenance evidence only and
is not a contribution or result of this conformal-set paper.

\end{document}